\documentclass[12pt]{article}
\usepackage{amsmath,amssymb}
\usepackage{amsthm}
\edef\,{\thinspace} \edef\;{\thickspace} \edef\!{\negthinspace} 
\def\dispmuskip{\thinmuskip= 3mu plus 0mu minus 2mu \medmuskip=  4mu plus 2mu minus 2mu \thickmuskip=5mu plus 5mu minus 2mu}
\def\textmuskip{\thinmuskip= 0mu                    \medmuskip=  1mu plus 1mu minus 1mu \thickmuskip=2mu plus 3mu minus 1mu}
\textmuskip
\def\beq{\dispmuskip\begin{equation}}    \def\eeq{\end{equation}\textmuskip}
\def\beqn{\dispmuskip\begin{displaymath}}\def\eeqn{\end{displaymath}\textmuskip}
\def\bqa{\dispmuskip\begin{eqnarray}}    \def\eqa{\end{eqnarray}\textmuskip}
\def\bqan{\dispmuskip\begin{eqnarray*}}  \def\eqan{\end{eqnarray*}\textmuskip}
\def\paradot#1{\vspace{1.3ex plus 0.5ex minus 0.5ex}\noindent{\bf{#1.}}}
\def\paranodot#1{\vspace{1.3ex plus 0.5ex minus 0.5ex}\noindent{\bf{#1}}}

\newtheorem{theorem}{Theorem}
\newtheorem{corollary}[theorem]{Corollary}

\newtheorem{proposition}[theorem]{Proposition}
\newtheorem{claim}[theorem]{Claim}
\newtheorem{definition}[theorem]{Definition}
\theoremstyle{remark}
\newtheorem*{note*}{Note}

\newenvironment{keywords}%
  {\centerline{\bf\small Keywords}\begin{quote}\small}%
  {\par\end{quote}\vskip 1ex}

\let\phi\varphi
\def\as{\text{ a.s.}}

\def\P{\operatorname{\bf P}}
\def\E{\operatorname{\bf E}}
\def\v{\boldsymbol}
\def\up{\overline}
\def\low{\underline}
\def\r#1#2{r_{#1..#2}}

\def\odm{{\textstyle{1\over m}}}
\def\SetR{I\!\!R}
\def\SetN{I\!\!N}
\def\C{{\cal C}}                        
\def\X{{\cal X}}                        
\def\Y{{\cal Y}}                        
\def\R{{\cal R}}                        
\def\O{{\cal O}}                        
\def\Z{{\cal Z}}                        
\def\qmbox#1{{\quad\mbox{#1}\quad}}
\def\l{{\ell}}                          
\def\eps{\varepsilon}                   

\def\o{\omega}

\begin{document}

\title{\vspace{-3ex}\normalsize\sc Technical Report \hfill IDSIA-08-08
\vskip 2mm\bf\Large\hrule height5pt \vskip 6mm
On the Possibility of Learning in Reactive Environments with Arbitrary Dependence
\vskip 6mm \hrule height2pt}
\author{
\begin{minipage}{0.49\textwidth}\centering
{\bf Daniil Ryabko}\\[3mm]
\normalsize IDSIA, Galleria 2, CH-6928\\
Manno-Lugano, Switzerland%
\thanks{This work was supported by the Swiss NSF grants 200020-107616 and 200021-113364.
}
\\
daniil@idsia.ch
\end{minipage}\hfill
\begin{minipage}{0.49\textwidth}\centering
{\bf Marcus Hutter}\\[3mm]
\normalsize RSISE$\,$@$\,$ANU and SML$\,$@$\,$NICTA \\
\normalsize Canberra, ACT, 0200, Australia \\
http://www.hutter1.net
\end{minipage}
}
\date{October 2008}
\maketitle
\vspace{-7ex}
\begin{abstract}
We address the problem of reinforcement learning in which
observations may exhibit an arbitrary form of stochastic dependence
on past observations and actions, i.e.\ environments more general
than (PO)MDPs. The task for an agent is to attain the  best possible
asymptotic reward where the true generating environment is unknown
but belongs to a known countable family of environments. We find
some sufficient conditions on the class of  environments under which
an agent exists which attains the best asymptotic reward for any
environment in the class. We analyze how tight these conditions are
and how they relate to different probabilistic assumptions known in
reinforcement learning and related fields, such as Markov Decision
Processes and mixing conditions.
\def\contentsname{\centering\normalsize Contents}
{\parskip=-2.5ex\tableofcontents}
\end{abstract}

\begin{keywords}
Reinforcement learning,
asymptotic average value,
self-optimizing policies,
(non) Markov decision processes.
\end{keywords}

\section{Introduction}\label{secInt}

Many real-world ``learning'' problems (like learning to drive a car
or playing a game) can be modelled as an agent $\pi$ that interacts
with an environment $\mu$ and is (occasionally) rewarded for its
behavior. We are interested in agents which perform well in the
sense of having high long-term reward, also called the value
$V(\mu,\pi)$ of agent $\pi$ in environment $\mu$. If $\mu$ is known,
it is a pure (non-learning) computational problem to determine the
optimal agent $\pi^\mu:=\arg\max_\pi V(\mu,\pi)$. It is far less
clear what an ``optimal'' agent means, if $\mu$ is unknown.
A reasonable objective is to have a single policy $\pi$ with high
value simultaneously in many environments. We will formalize and
call this criterion {\em self-optimizing} later.

\paradot{Learning approaches in reactive worlds}
Reinforcement learning, sequential decision theory, adaptive
control theory, and active expert advice, are theories dealing
with this problem. They overlap but have different core focus:
Reinforcement learning algorithms \cite{Sutton:98} are developed to
learn $\mu$ or directly its value. Temporal difference learning is
computationally very efficient, but has slow asymptotic guarantees
(only) in (effectively) small observable MDPs. Others have faster
guarantee in finite state MDPs \cite{Brafman:01}. There are
algorithms \cite{EvenDar:05} which are optimal for any finite
connected POMDP, and this is apparently the largest class of
environments considered.
In sequential decision theory, a Bayes-optimal agent $\pi^*$ that
maximizes $V(\xi,\pi)$ is considered, where $\xi$ is a mixture of
environments $\nu\in\C$ and $\C$ is a class of environments that
contains the true environment $\mu\in\C$ \cite{Hutter:04uaibook}.
Policy $\pi^*$ is self-optimizing in an arbitrary (e.g.\ non-POMDP)
class $\C$, provided $\C$ allows for self-optimizingness
\cite{Hutter:02selfopt}.
Adaptive control theory \cite{Kumar:86} considers very simple
(from an AI perspective) or special systems (e.g.\ linear with
quadratic loss function), which sometimes allow computationally
and data efficient solutions.
Action with expert advice
\cite{Farias:03,Poland:05dule,Poland:05aixifoe,Cesa:06}
constructs an agent (called master) that performs nearly as well
as the best agent (best expert in hindsight) from some class of
experts, in {\em any} environment $\nu$.
The important special case of passive sequence prediction in
arbitrary unknown environments, where the actions=predictions do
not affect the environment is comparably easy
\cite{Hutter:03optisp,Hutter:04expert}.

The difficulty in active learning problems can be identified (at
least, for countable classes) with {\em traps} in the environments.
Initially the agent does not know $\mu$, so has asymptotically to be
forgiven in taking initial ``wrong'' actions. A well-studied such
class are ergodic MDPs which guarantee that, from any action
history, every state can be (re)visited \cite{Hutter:02selfopt}.

\paradot{What's new}
The aim of this paper is to characterize as general as possible
classes $\C$ in which self-optimizing behaviour is possible, more
general than POMDPs. To do this we need to characterize classes of
environments that forgive. For instance, exact state recovery is
unnecessarily strong; it is sufficient being able to recover high
rewards, from whatever states. Further, in many real world problems
there is no information available about the ``states'' of the
environment (e.g.\ in POMDPs) or the environment may exhibit long
history dependencies.

Rather than trying to model an environment (e.g. by MDP) we try to
identify the conditions sufficient for learning. Towards this aim,
we propose to consider only environments in which, after any
arbitrary finite sequence of actions, the best value is still
achievable. The performance criterion here is asymptotic average
reward. Thus we consider such environments for which there exists a
policy whose asymptotic average reward exists and upper-bounds
asymptotic average reward of any other policy. Moreover, the same
property should hold after any finite sequence of actions has been
taken (no traps). We call such environments \emph{recoverable}. If
we only want to get $\eps$-close to the optimal value infinitely
often with decreasing $\eps$ (that is, to have the same upper limit
for the average value), then this property is already sufficient.

Yet recoverability in itself is not sufficient for identifying
behaviour which results in optimal limiting average value. We
require further that, from any sequence of $k$ actions, it is
possible to return to the optimal level of reward in $o(k)$ steps;
that is, it is not just possible to recover after any sequence of
(wrong) actions, but it is possible to recover fast. Environments
which possess this property are called \emph{value-stable}. (These
conditions will be formulated in a probabilistic form.)

We show that for any countable class of value-stable environments
there exists a policy which achieves the best possible value in any
of the environments from the class (i.e. is \emph{self-optimizing}
for this class).

Furthermore, we present some examples of environments which possess
value-stability and/or recoverability. In particular, any ergodic
MDP can be easily shown to be value-stable. A mixing-type condition
which implies value-stability is also demonstrated. In addition, we
provide a construction allowing to build examples of value-stable
and/or recoverable environments which are not isomorphic to a finite
POMDP, thus demonstrating that the class of value-stable
environments is quite general.

Finally, we consider environments which are not recoverable but
still are value-stable. In other words, we consider the question of
what it means to be optimal in an environment which does not
``forgive'' wrong actions. Even in such cases some policies are
better than others, and we identify some conditions which are
sufficient for learning a policy that is optimal from some point on.

It is important in our argument that the class of environments for
which we seek a self-optimizing policy is countable, although the
class of all value-stable environments is uncountable. To find a set
of conditions necessary and sufficient for learning which do not
rely on countability of the class is yet an open problem. However,
from a computational perspective countable classes are sufficiently
large (e.g.\ the class of all computable probability measures is
countable).

\paradot{Contents}
The paper is organized as follows. Section~\ref{secNot} introduces
necessary notation of the agent framework. In Section~\ref{secSetup}
we define and explain the notion of value-stability, which is
central to the paper, and a weaker but simpler notion of
recoverability. Section~\ref{secMain} presents the theorems about
self-optimizing policies for classes of value-stable environments
and recoverable environments. In Section~\ref{sec:nonrec} we
discuss what can be achieved if the environments are not
recoverable. Section~\ref{sec:ex} illustrates the applicability of
the theorems by providing examples of value-stable and recoverable
environments. In Section~\ref{sec:nec} we discuss necessity of the
conditions of the main theorems. Section~\ref{secDisc} provides some
discussion of the results and an outlook to future research. Formal
proofs of the main theorems are given in Section~\ref{secPrThMain},
while Sections~\ref{secMain} and~\ref{sec:nonrec} contain only
intuitive explanations.

\section{Notation and Definitions}\label{secNot}

We essentially follow the notation of
\cite{Hutter:02selfopt,Hutter:04uaibook}.

\paradot{Strings and probabilities}
We use letters $i,k,l,m,n\in\SetN$ for natural numbers, and denote
the cardinality of sets $\cal S$ by $\#{\cal S}$. We write $\X^*$
for the set of finite strings over some alphabet $\X$, and
$\X^\infty$ for the set of infinite sequences. For a string
$x\in\X^*$ of length $\l(x)=n$ we write $x_1x_2...x_n$ with
$x_t\in\X$ and further abbreviate $x_{k:n}:=x_kx_{k+1}...x_{n-1}x_n$
and $x_{<n}:=x_1... x_{n-1}$. Finally, we define
$x_{k..n}:=x_k+...+x_n$, provided elements of $\X$ can be added.

We assume that sequence $\o=\o_{1:\infty}\in\X^\infty$ is sampled
from the ``true'' probability measure $\mu$, i.e.\
$\P[\o_{1:n}=x_{1:n}]=\mu(x_{1:n})$. We denote expectations
w.r.t.\ $\mu$ by $\E$, i.e.\ for a function $f:\X^n\to\SetR$,
$\E[f]=\E[f(\o_{1:n})]=\sum_{x_{1:n}}\mu(x_{1:n})f(x_{1:n})$.
When we use probabilities
and expectations with respect to other measures we make the
notation explicit, e.g. $\E_\nu$ is the expectation with respect
to $\nu$.
Measures $\nu_1$ and $\nu_2$ are called {\em singular} if there
exists a set $A$ such that $\nu_1(A)=0$ and $\nu_2(A)=1$.

\paranodot{The agent framework}
is general enough to allow modelling nearly any kind of
(intelligent) system \cite{Russell:95}.
In cycle $k$, an agent performs {\em action} $y_k\in\Y$ (output)
which results in {\em observation} $o_k\in\O$ and {\em reward}
$r_k\in\R$, followed by cycle $k+1$ and so on.
We assume that the action space $\Y$, the observation space $\O$,
and the reward space $\R\subset\SetR$ are finite, w.l.g.
$\R=\{0,\dots,r_{max}\}$.
We abbreviate $z_k:=y_k r_k o_k\in\Z:=\Y\times\R\times\O$ and
$x_k=r_k o_k\in\X:=\R\times\O$.
An agent is identified with a (probabilistic) {\em policy} $\pi$.
Given {\em history} $z_{<k}$, the probability that agent $\pi$ acts
$y_k$ in cycle $k$ is (by definition) $\pi(y_k|z_{<k})$. Thereafter,
{\em environment} $\mu$ provides (probabilistic) reward $r_k$ and
observation $o_k$, i.e.\ the probability that the agent perceives
$x_k$ is (by definition) $\mu(x_k|z_{<k}y_k)$.
Note that the policy and the environment are allowed to depend on
the complete history. We do not make any MDP or POMDP assumption
here, and we don't talk about states of the environment, only about
observations.
Each (policy,environment) pair $(\pi,\mu)$ generates an I/O
sequence $z_1^{\pi\mu}z_2^{\pi\mu}...$. Mathematically,
the history $z_{1:k}^{\pi\mu}$ is a random variable with probability
\beqn
  \P\big(z_{1:k}^{\pi\mu}=z_{1:k}\big) \;=\; \pi(y_1)\cdot\mu(x_1|y_1)\cdot
    ... \cdot\pi(y_k|z_{<k})\cdot\mu(x_k|z_{<k}y_k).
\eeqn
Since value maximizing policies can always be chosen
deterministic, there is no real need to consider probabilistic
policies, and henceforth we consider deterministic policies $p$.
We assume that $\mu\in\C$ is the true, but unknown, environment,
and $\nu\in\C$ a generic environment.

\section{Setup}\label{secSetup}

For an environment $\nu$ and a policy $p$ define
random variables (upper and lower average value)
\beqn
  \up V(\nu,p) \;:=\; \limsup_m\left\{\odm\r1m^{p\nu}\right\} \qmbox{and}
  \low V(\nu,p) \;:=\; \liminf_m\left\{\odm\r1m^{p\nu}\right\}
\eeqn
where $\r1m:=r_1+...+r_m$.
If there exists a constant $\up V$ or a constant $\low V$  such that
\beqn
  \up V(\nu,p) \;=\; \up V\as \text{, or }\; \low V(\nu,p) \;=\; \low V\as
\eeqn
then we say that the upper limiting average or (respectively) lower
average value exists, and denote it by $\up V(\nu,p):=\up V$ (or
$\low V(\nu,p):=\low V$). If both upper and lower average limiting
values exist and are equal then we simply say that average limiting
value exist and denote it by $V(\nu,p):=\up V(\nu,p)=\low V(\nu,p)$

An environment $\nu$ is \emph{explorable} if there exists a policy
$p_\nu$ such that $V(\nu,p_\nu)$ exists and $\up V(\nu,p)\le
V(\nu,p_\nu)$ with probability 1 for every policy $p$. In this
case define
$V^*_\nu:=V(\nu,p_\nu)$.
An environment $\nu$ is \emph{upper explorable} if there exists a policy
$p_\nu$ such that $\up V(\nu,p_\nu)$ exists and $\up V(\nu,p)\le
\up V(\nu,p_\nu)$ with probability 1 for every policy $p$. In this
case define
$\up V^*_\nu:=\up V(\nu,p_\nu)$.

A policy $p$ is \emph{self-optimizing} for a set of explorable
environments $\C$ if $V(\nu,p)=V^*_\nu$ for every $\nu\in \C$. A
policy $p$ is \emph{upper self-optimizing} for a set of explorable
environments $\C$ if $\up V(\nu,p)=\up V^*_\nu$ for every $\nu\in
\C$.

In the case when we we wish to obtain the optimal average value for
any environment in the class we will speak about self-optimizing
policies, whereas if we are only interested in obtaining the upper
limit of the average value then we will speak about upper
self-optimizing policies. It turns out that the latter case is much
simpler. The next two definitions present conditions on the
environments which will be shown to be sufficient to achieve the two
respective goals.

\begin{definition}[recoverable]\label{def:vstable2}
We call an upper explorable environment $\nu$ recoverable if for
any history $z_{<k}$ such that $\nu(x_{<k}|y_{<k})>0$ there exists a
policy $p$ such that
$$
  \P (\up V(\nu,p) =\up V^*| z_{<k} )=1.
$$
\end{definition}

Conditioning on the history $z_{<k}$ means that we take
$\nu$-conditional probabilities (conditional on $x_{<k}$) and first
$k-1$ actions of the policy $p$ are replaced by $y_{<k}$.

Recoverability means that after taking any finite sequence of
(possibly sub-optimal) actions it is still possible to obtain the
same upper limiting average value as an optimal policy would obtain.
The next definition is somewhat more complex.

\begin{definition}[value-stable environments]\label{def:vstable}
An explorable environment $\nu$ is \emph{value-stable}
if there exist a sequence of numbers $r^\nu_i\in[0,r_{max}]$ and
two functions $d_\nu(k,\eps)$ and $\phi_\nu(n,\eps)$ such that
$\frac{1}{n} \r1n^\nu\rightarrow V^*_\nu$, $d_\nu(k,\eps)=o(k)$,
$\sum_{n=1}^\infty\phi_\nu(n,\eps)<\infty$ for every fixed $\eps$,
and for every $k$ and every history $z_{<k}$ there exists a policy
$p=p_\nu^{z_{<k}}$ such that
\beq\label{eq:svs}
  \P\left(\r{k}{k+n}^\nu-\r{k}{k+n}^{p\nu}
          > d_\nu(k,\eps)+n\eps\mid z_{<k}\right)
  \le \phi_\nu(n,\eps).
\eeq
\end{definition}

First of all, this condition means that the strong law of large
numbers for rewards holds uniformly over histories $z_{<k}$; the
numbers $r^\nu_i$ here can be thought of as expected rewards of an
optimal policy. Furthermore, the environment is ``forgiving'' in the
following sense: from any (bad) sequence of $k$ actions it is
possible (knowing the environment) to recover up to $o(k)$ reward
loss; to recover means to reach the level of reward obtained by the
optimal policy which from the beginning was taking only optimal
actions. That is, suppose that a person A has made $k$ possibly
suboptimal actions and after that ``realized'' what the true
environment was and how to act optimally in it. Suppose that a
person B was from the beginning taking only optimal actions. We want
to compare the performance of A and B on first $n$ steps after the
step $k$. An environment is value stable if A can catch up with B
except for $o(k)$ gain. The numbers $r_i^\nu$ can be thought of as
expected rewards of B; A can catch up with B up to the reward loss
$d_\nu(k,\eps)$ with probability $\phi_\nu(n,\eps)$, where the
latter does not depend on past actions and observations (the law of
large numbers holds uniformly).

Examples of value-stable environments will be considered in
Section~\ref{sec:ex}.

\section{Main Results}\label{secMain}

In this section we present the main self-optimizingness result along
with an informal explanation of its proof, and a result on upper
self-optimizingness, which turns out to have much more simple
conditions.

\begin{theorem}[value-stable$\Rightarrow$self-optimizing]\label{th:main}
For any countable class $\C$ of value-stable
environments, there exists a policy which is self-optimizing for
$\C$. \end{theorem}

A formal proof is given in the appendix; here we give some intuitive
justification. Suppose that all environments in $\C$ are
deterministic. We will construct a self-optimizing policy $p$ as
follows: Let $\nu^t$ be the first environment in $\C$. The algorithm
assumes that the true environment is $\nu^t$ and tries to get
$\eps$-close to its optimal value for some (small) $\eps$. This is
called an exploitation part. If it succeeds, it does some
exploration as follows. It picks the first environment $\nu^e$ which
has higher average asymptotic value than $\nu^t$
($V^*_{\nu^e}>V^*_{\nu^t}$) and tries to get $\eps$-close to this
value acting optimally under $\nu^e$. If it cannot get close to the
$\nu^e$-optimal value then $\nu^e$ is not the true environment, and
the next environment can be picked for exploration (here we call
``exploration'' successive attempts to exploit an environment which
differs from the current hypothesis about the true environment and
has a higher average reward). If it can, then it switches to
exploitation of $\nu^t$, exploits it until it is $\eps'$-close to
$V^*_{\nu^t}$, $\eps'<\eps$ and switches to $\nu^e$ again this time
trying to get $\eps'$-close to $V_{\nu^e}$; and so on. This can
happen only a finite number of times if the true environment is
$\nu^t$, since $V^*_{\nu^t}<V^*_{\nu^e}$. Thus after exploration
either $\nu^t$ or $\nu^e$ is found to be inconsistent with the
current history. If it is $\nu^e$ then just the next environment
$\nu^e$ such that $V^*_{\nu^e}>V^*_{\nu^t}$ is picked for
exploration. If it is $\nu^t$ then the first consistent environment
is picked for exploitation (and denoted $\nu^t$). This in turn can
happen only a finite number of times before the true environment
$\nu$ is picked as $\nu^t$. After this, the algorithm still
continues its exploration attempts, but can always keep within
$\eps_k\rightarrow0$ of the optimal value. This is ensured by
$d(k)=o(k)$.

The probabilistic case is somewhat more complicated since we can not
say whether an environment is ``consistent'' with the current
history. Instead we test each environment for consistency as
follows. Let $\xi$ be a mixture of all environments in $\C$. Observe
that together with some fixed policy each environment $\mu$ can be
considered as a measure on $\Z^\infty$. Moreover, it can be shown
that (for any fixed policy) the ratio
$\frac{\nu(z_{<n})}{\xi(z_{<n})} $ is bounded away from zero if
$\nu$ is the true environment $\mu$ and tends to zero if $\nu$ is
singular with $\mu$ (in fact, here singularity is a probabilistic
analogue of inconsistency). The exploration part of the algorithm
ensures that at least one of the environments $\nu^t$ and $\nu^e$ is
singular with $\nu$ on the current history, and a succession of
tests $\frac{\nu(z_{<n})}{\xi(z_{<n})}\ge\alpha_s$ with
$\alpha_s\rightarrow0$ is used to exclude such environments from
consideration.

\paradot{Upper self-optimizingness}
Next we consider the task in which our goal is more moderate. Rather
than trying to find a policy which will obtain the same average
limiting value as an optimal one for any environment in a certain
class, we will try to obtain only the optimum upper limiting
average. That is, we will try to find a policy which infinitely
often gets as close as desirable to the maximum possible average
value. It turns out that in this case a much simpler condition is
sufficient: recoverability instead of value-stability.

\begin{theorem}[recoverable$\Rightarrow$upper self-optimizing]\label{th:wr}
For any countable class $\C$ of recoverable
environments, there exists a policy which is upper self-optimizing for
$\C$.
\end{theorem}
A formal proof can be found in Section~\ref{secPrThMain}; its idea
is as follows. The upper self-optimizing policy $p$ to be
constructed will loop through all environments in $\C$ in such a way
that each environment is tried infinitely often, and for each
environment the agent will try to get $\eps$-close (with decreasing
$\eps$) to the upper-limiting average value, until it either manages
to do so, or a special stopping condition holds:
$\frac{\nu(z_{<n})}{\xi(z_{<n})}<\alpha_s$, where $\alpha_s$ is
decreasing accordingly. This condition necessarily breaks if the
upper limiting average value cannot be achieved.

\section{Non-recoverable environments}\label{sec:nonrec}

Before proceeding with examples of value-stable environments, we
briefly discuss what can be achieved if an environment does not
forgive initial wrong actions, that is, is not recoverable. It turns
out that value-stability can be defined for non-recoverable
environments as well, and optimal~--- in a worst-case sense~---
policies can be identified.

For an environment $\nu$, a policy $p$ and a history $z_{<k}$ such
that $\nu(x_{<k}|y_{<k})>0$, if there exists a constant $\up V$ or a
constant $\low V$ such that
$$
  P(\up V(\nu,p) \;=\; \up V |z_{<k})=1 \text{, or }\;
  P(\low V(\nu,p) \;=\; \low V |z_{<k})=1,
$$
then we say that the upper conditional (on $z_{<k}$) limiting
average or (respectively) lower conditional average value exists,
and denote it by $\up V(\nu,p,z_{<k}):=\up V$ (or $\low
V(\nu,p,z_{<k}):=\low V$). If both upper and lower conditional
average limiting values exist and are equal then we say that that
average conditional value exist and denote it by
$V(\nu,p,(z_{<k})):=\up V(\nu,p,z_{<k})=\low V(\nu,p,z_{<k})$

Call an environment $\nu$ \emph{strongly (upper) explorable} if for
any history $z_{<k}$ such that $\nu(x_{<k}|y_{<k})>0$ there exists a
policy $p_\nu^{z_{<k}}$ such that $V(\nu,p_\nu^{z_{<k}})$ ($\up
V(\nu,p_\nu^{z_{<k}})$) exists and $\up V(\nu,p,z_{<k})\le
V(\nu,p^{z_{<k}}_\nu,z_{<k})$ (respectively $\up V(\nu,p,z_{<k})\le
\up V(\nu,p^{z_{<k}}_\nu,z_{<k})$) with probability 1 for every
policy $p$. In this case define
$V^*_\nu(z_{<k}):=V(\nu,p^{z_{<k}}_\nu)$ (respectively $\up
V^*_\nu(z_{<k}):=\up V(\nu,p^{z_{<k}}_\nu)$).

For a strongly explorable environment $\nu$ define the worst-case
optimal value
$$
W^*_\nu:= \inf_{k, z_{<k}: \nu(x_{<k}>0)} V^*_\nu(z_{<k}),
$$
and for a strongly upper explorable $\nu$ define the worst-case
upper optimal value
$$
\up W^*_\nu:= \inf_{k, z_{<k}: \nu(x_{<k}>0)}\up V^*_\nu(z_{<k}).
$$
In words, the worst-case optimal value is the asymptotic average
reward which is attainable with certainty after any finite sequence
of actions has been taken.

Note that a recoverable explorable environment is also strongly explorable.

A policy $p$ will be called {\em worst-case self-optimizing} or {\em
worst-case upper self-optimizing} for a class of environments
$\mathcal C$ if $\liminf{1\over m}r_{1..m}^{p\nu} \ge W^*_\nu$, or
(respectively) $\limsup {1\over m}r_{1..m}^{p\nu} \ge \up W^*_\nu$
with probability 1 for every $\nu\in\mathcal C$.

\begin{definition}[worst-case value-stable environments]\label{def:wstable}
A strongly explorable environment $\nu$ is \emph{worst-case value-stable}
if there exists a sequence of numbers $r^\nu_i\in[0,r_{max}]$ and
two functions $d_\nu(k,\eps)$ and $\phi_\nu(n,\eps)$ such that
$\frac{1}{n} \r1n^\nu\rightarrow W^*_\nu$, $d_\nu(k,\eps)=o(k)$,
$\sum_{n=1}^\infty\phi_\nu(n,\eps)<\infty$ for every fixed $\eps$,
and for every $k$ and every history $z_{<k}$ there exists a policy
$p=p_\nu^{z_{<k}}$ such that
\beq\label{eq:svs2}
\P\left(
 \r{k}{k+n}^\nu - \r{k}{k+n}^{p\nu} >
 d_\nu(k,\eps)+n\eps\mid z_{<k}\right)\le \phi_\nu(n,\eps).
\eeq
\end{definition}

Note that a recoverable environment is value-stable if and only if
it is worst-case value-stable.

Worst-case value stability helps to distinguish between irreversible
actions (or ``traps'') and actions which result only in a temporary
loss in performance; moreover, worst-case value-stability means that
a temporary loss in performance can only be short (sublinear).

Finally, we can establish the following result (cf. Theorems
\ref{th:main} and~\ref{th:wr}).
\begin{theorem}[worst-case self-optimizing]\label{th:worst}
\begin{itemize}
\item[(i)]  For any countable set of worst-case value-stable
            environments $\mathcal C$ there exist a policy $p$
            which is worst-case self-optimizing for $\mathcal C$.
\item[(ii)] For any countable set of strongly upper explorable
            environments $\mathcal C$ there exist a policy $p$
            which is worst-case upper self-optimizing for $\mathcal C$.
\end{itemize}
\end{theorem}
The proof of this theorem is analogous to the proofs of
Theorems~\ref{th:main} and~\ref{th:wr}; the differences are
explained in Section~\ref{secPrThMain}.

\section{Examples}\label{sec:ex}

In this section we illustrate the results of the previous section
with examples of classes of value-stable environments. These are
also examples of recoverable environments, since recoverability
is strictly weaker than value-stability. In the end of the section
we also give some simple examples of recoverable but not
value-stable environments.

We first note that passive environments are value-stable. An
environment is called \emph{passive} if the observations and rewards
do not depend on the actions of the agent. Sequence prediction task
provides a well-studied (and perhaps the only reasonable) class of
passive environments: in this task the agent is required to give the
probability distribution of the next observation given the previous
observations. The true distribution of observations depends only on
the previous observations (and does not depend on actions and
rewards). Since we have confined ourselves to considering finite
action spaces, the agent is required to give ranges of probabilities
for the next observation, where the ranges are fixed beforehand. The
reward $1$ is given if all the ranges are correct and the reward $0$
is given otherwise. It is easy to check that any such environment is
value-stable with $r_i^\nu\equiv 1$, $d(k,\eps)\equiv1$,
$\phi(n,\eps)\equiv0$, since, knowing the distribution, one can
always start giving the correct probability ranges (this defines the
policy $p_\nu$).

Obviously, there are active value stable environments too. The next
proposition provides some conditions on mixing rates which are
sufficient for value-stability; we do not intend to provide sharp
conditions on mixing rates but rather to illustrate the relation of
value-stability with mixing conditions.

We say that a stochastic process $h_k$, $k\in\SetN$ satisfies
strong $\alpha$-mixing conditions with coefficients $\alpha(k)$
if (see e.g. \cite{Bosq:96})
 \beqn
 \sup_{n\in\SetN} \sup_{B\in\sigma(h_1,\dots,h_n), C\in\sigma(h_{n+k},\dots)}
 |\P(B\cap C)-\P(B)\P(C)| \le \alpha(k),
\eeqn
where $\sigma()$ stands for the sigma-algebra generated by the
random variables in brackets. Loosely speaking, mixing
coefficients $\alpha$ reflect the speed with which the process
``forgets'' about its past.

\begin{proposition}[mixing and value-stability]\label{prop:mix}
Suppose that an explorable environment $\nu$ is such that there
exist a sequence of numbers $r^\nu_i$ and a function $d(k)$ such
that $\frac{1}{n}\r1n^\nu\rightarrow V^*_\nu$, $d(k)=o(k)$, and for
each $z_{<k}$ there exists a policy $p$ such that the sequence
$r_i^{p\nu}$ satisfies strong $\alpha$-mixing conditions with
coefficients $\alpha(k)=\frac{1}{k^{1+\eps}}$ for some $\eps>0$ and
\beqn
  \r{k}{k+n}^\nu-\E \left( \r{k}{k+n}^{p\nu}\mid z_{<k}\right)\le d(k)
\eeqn
for any $n$. Then $\nu$ is value-stable.
\end{proposition}

\begin{proof}
Using the union bound we obtain
\bqan
  & & \P\left( \r{k}{k+n}^\nu-\r{k}{k+n}^{p\nu}>d(k)+n\eps\right)
\\
  & & \le I\left( \r{k}{k+n}^\nu-\E\r{k}{k+n}^{p\nu}>d(k)\right)+
    \P\left(\left|\r{k}{k+n}^{p\nu}-\E\r{k}{k+n}^{p\nu}\right|>n\eps\right).
\eqan
The first term equals $0$ by assumption and the second term for each
$\eps$ can be shown to be summable using \cite[Thm.1.3]{Bosq:96}:
for a sequence of uniformly bounded zero-mean random variables $r_i$
satisfying strong $\alpha$-mixing conditions the following bound
holds true for any integer $q\in[1,n/2]$
\beqn
  \P\left(|\r1n|>n\eps\right)\le c e^{-\eps^2q/c}+cq\alpha\left(\frac{n}{2q}\right)
\eeqn
for some constant $c$; in our case we just set
$q=n^{\frac{\eps}{2+\eps}}$.
\end{proof}

\paradot{(PO)MDPs}
Applicability of Theorem~\ref{th:main} and
Proposition~\ref{prop:mix} can be illustrated on (PO)MDPs. We note
that self-optimizing policies for (uncountable) classes of finite
ergodic MDPs and POMDPs are known \cite{Brafman:01,EvenDar:05}; the
aim of the present section is to show that value-stability is a
weaker requirement than the requirements of these models, and also
to illustrate applicability of our results. We call $\mu$ a
(stationary) {\em Markov decision process} (MDP) if the probability
of perceiving $x_k\in\X$, given history $z_{<k}y_k$ only depends on
$y_k\in\Y$ and $x_{k-1}$. In this case $x_k\in\X$ is called a {\em
state}, $\X$ the {\em state space}.
An MDP $\mu$ is called {\em ergodic} if there exists a policy
under which every state is visited infinitely often with
probability 1. An MDP with a stationary policy forms a Markov
chain.

An environment is called a (finite) {\em partially observable MDP}
(POMDP) if there is a sequence of random variables $s_k$ taking
values in a finite space $\mathcal S$ called the state space, such
that $x_k$ depends only on $s_k$ and $y_k$, and $s_{k+1}$ is
independent of $s_{<k}$ given $s_k$. Abusing notation the
sequence $s_{1:k}$ is called the underlying Markov chain. A POMDP
is called {\em ergodic} if there exists a policy such that the
underlying Markov chain visits each state infinitely often with
probability 1.

In particular, any ergodic POMDP $\nu$ satisfies strong
$\alpha$-mixing conditions with coefficients decaying exponentially
fast in case there is a set $H\subset\R$ such that $\nu(r_i\in H)=1$
and $\nu(r_i=r|s_i=s,y_i=y)\ne 0$ for each $y\in\Y, s\in \mathcal S,
r\in H, i\in\SetN$. Thus for any such POMDP $\nu$ we can use
Proposition \ref{prop:mix} with $d(k,\eps)$ a constant function to
show that $\nu$ is value-stable:
\begin{corollary}[POMDP$\Rightarrow$value-stable] Suppose that a
POMDP $\nu$ is ergodic and there exists a set $H\subset\R$ such
that $\nu(r_i\in H)=1$ and $\nu(r_i=r|s_i=s,y_i=y)\ne 0$ for each $y\in\Y, h\in
\mathcal S, r\in H$, where $\mathcal S$ is the finite state space
of the underlying Markov chain. Then $\nu$ is
value-stable.
\end{corollary}

However, it is illustrative to obtain this result
for MDPs directly, and in a slightly stronger form.

\begin{proposition}[MDP$\Rightarrow$value-stable]\label{prop:MDP}
Any finite-state ergodic MDP $\nu$ is a
value-stable environment.
\end{proposition}
\begin{proof}
Let $d(k,\eps)=0$. Denote by $\mu$ the
true environment, let $z_{<k}$ be the current history and let the
current state (the observation $x_{k}$) of the environment be
$a\in\X$, where $\X$ is the set of all possible states. Observe
that for an MDP there is an optimal policy which depends only on
the current state. Moreover, such a policy is optimal for any
history. Let $p_\mu$ be such a policy. Let $r_i^\mu$ be the
expected reward of $p_\mu$ on step $i$. Let $l(a,b)=\min\{n:
x_{k+n}=b | x_{k}=a \}$. By ergodicity of $\mu$ there exists a
policy $p$ for which $\E l(b,a)$ is finite (and does not depend
on $k$). A policy $p$ needs to get from the state $b$ to one of
the states visited by an optimal policy, and then acts according
to $p_\mu$. Let $f(n):=\frac{nr_{\max}}{\log n}$. We have \bqan
  & & \P\left( \left|\r{k}{k+n}^\mu-\r{k}{k+n}^{p\mu}\right|>n\eps\right)\le
  \sup_{a\in\X } \P\left( \left|\E\left(\r{k}{k+n}^{p_\mu\mu}|x_k=a\right)
           - \r{k}{k+n}^{p\mu}\right|>n\eps)\right)
\it \\ & & \le \sup_{a,b\in\X}\P(l(a,b)>f(n)/r_{\max})
\\
 & & +
    \; \sup_{a,b\in\X}\P\left( \left|\E\left(\r{k}{k+n}^{p_\mu\mu}| x_k=a\right)
           - \r{k+f(n)}{k+n}^{p_\mu\mu}\right|>n\eps-f(n)\Big|
           x_{k+f(n)}=a\right)
\rm \\ & & \le \sup_{a,b\in\X}\P(l(a,b)>f(n)/r_{\max})
\\
 & & +
    \; \sup_{a\in\X}\P\left( \left|\E\left(\r{k}{k+n}^{p_\mu\mu}| x_k=a\right)
           - \r{k}{k+n}^{p_\mu\mu}\right|>n\eps-2f(n)\Big|
           x_{k}=a\right).
\eqan
In the last term we have the deviation of the reward attained by
the optimal policy from its expectation. Clearly, both terms are
bounded exponentially in $n$.
\end{proof}

In the examples above the function $d(k,\eps)$ is a constant and
$\phi(n,\eps)$ decays exponentially fast. This suggests that the
class of value-stable environments stretches beyond finite
(PO)MDPs. We illustrate this guess by the construction that follows.

A general scheme for constructing
{\bf value-stable environment or recoverable environments}:
infinitely armed bandit.
Next we present a construction of environments which cannot be
modelled as finite POMDPs but are value-stable and/or recoverable.
Consider the following environment $\nu$. There is a countable
family $\C'=\{\zeta_i: i\in\SetN\}$ of {\em arms}, that is, sources
generating i.i.d. rewards $0$ and $1$ (and, say, empty observations)
with some probability $\delta_i$ of the reward being $1$. The action
space $\Y$ consists of three actions $\Y=\{g,u,d\}$. To get the next
reward from the current arm $\zeta_i$ an agent can use the action
$g$. Let $i$ denote the index of the current arm. At the beginning
$i=0$, the current arm is $\zeta_0$ and then the agent can move
between arms as follows: it can move $U(i)$ arms ``up'' using the
action $u$ (i.e. $i:=i+U(i)$) or it can move $D(i)$ arms ``down''
using the action $d$ (i.e. $i:=i-D(i)$ or 0 if the result is
negative). The reward for actions $u$ and $d$ is $0$. In all the
examples below $U(i)\equiv1$, that is, the action $u$ takes the
agent one arm up.

Clearly, $\nu$ is a POMDP with countably infinite number of states
in the underlying Markov chain, which (in general) is not isomorphic
to a finite POMDP.

\begin{claim}
If $D(i)=i$ for all $i\in\SetN$ then the environment $\nu$ just
constructed is value-stable. If $D(i)\equiv1$ then $\nu$ is
recoverable but not necessarily value-stable; that is, there are
choices of the probabilities $\delta_i$ such that $\nu$ is not
value-stable.
\end{claim}
\begin{proof}
First we show that in either case ($D(i)=i$ or $D(i)=1$) $\nu$ is
explorable. Let $\delta=\sup_{i\in\SetN}\delta_i$. Clearly, $\up
V(\nu,p')\le \delta$ with probability $1$ for any policy $p'$ . A
policy $p$ which, knowing all the probabilities $\delta_i$, achieves
$\up V(\nu,p) =\low V(\nu,p) =\delta=:V^*_\nu$ a.s., can be easily
constructed. Indeed, find a sequence $\zeta'_j$, $j\in\SetN$, where
for each $j$ there is $i=:i_j$ such that $\zeta'_j=\zeta_i$,
satisfying $\lim_{j\rightarrow\infty}\delta_{i_j}=\delta$. The
policy $p$ should carefully exploit one by one the arms $\zeta_j$,
staying with each arm long enough to ensure that the average reward
is close to the expected reward with $\eps_j$ probability, where
$\eps_j$ quickly tends to 0, and so that switching between arms has
a negligible impact on the average reward. Thus $\nu$ can be shown
to be explorable. Moreover, a policy $p$ just sketched can be made
independent on (observation and) rewards.

Next we show if $D(i)=i$, that is, the action $d$ always takes the
agent down to the first arm, then the environment is value-stable.
Indeed, one can modify the policy $p$ (possibly allowing it to
exploit each arm longer) so that on each time step $t$ (from some
$t$ on) we have $j(t)\le\sqrt{t}$, where $j(t)$ is the number of the
current arm on step $t$. Thus, after any actions-perceptions history
$z_{<k}$ one needs about $\sqrt{k}$ actions (one action $u$ and
enough actions $d$) to catch up with $p$. So, (\ref{eq:svs}) can be
shown to hold with $d(k,\eps)=\sqrt{k}$, $r_i$ the expected reward
of $p$ on step $i$ (since $p$ is independent of rewards,
$r^{p\nu}_i$ are independent), and the rates $\phi(n,\eps)$
exponential in $n$.

To construct a non-value-stable environment with $D(i)\equiv1$,
simply set $\delta_0=1$ and $\delta_j=0$ for $j>0$; then after
taking $n$ actions $u$ one can only return to optimal rewards with
$n$ actions ($d$), that is $d(k)=o(k)$ cannot be obtained. Still it
is easy to check that recoverability is preserved, whatever the
choice of $\delta_i$.
\end{proof}

In the above construction we can also allow the action $d$ to bring
the agent $d(i)<i$ steps down, where $i$ is the number of the
current environment $\zeta$, according to some (possibly randomized)
function $d(i)$, thus changing the function $d_\nu(k,\eps)$ and
possibly making it non-constant in $\eps$ and as close as desirable
to linear.

\section{Necessity of value-stability}\label{sec:nec}

Now we turn to the question of how tight the conditions of
value-stability are. The following proposition shows that the
requirement $d(k,\eps)=o(k)$ in (\ref{eq:svs}) cannot be relaxed.

\begin{proposition}[\boldmath necessity of $d(k,\eps)=o(k)$]\label{th:tight}
There exists a countable family of deterministic explorable
environments $\C$ such that
\begin{itemize} \item for any $\nu\in\C$ for any sequence of
actions $y_{<k}$ there exists a policy $p$ such that
\beqn
  \r{k}{k+n}^\nu \le \r{k}{k+n}^{p\nu}+k \text{ for all }n\ge k,
\eeqn
where $r_i^\nu$ are the rewards attained by an optimal policy
$p_\nu$ (which from the beginning was acting optimally), but
\item for any policy $p$ there exists an environment $\nu\in\C$
such that $\low V(\nu,p) < V^*_\nu$ (i.e. there is no
self-optimizing policy for $\mathcal C$).
\end{itemize}
\end{proposition}
Clearly, each environment from such a class $\C$ satisfies the value
stability conditions with $\phi(n,\eps)\equiv0$ except
$d(k,\eps)=k\ne o(k)$.

\begin{proof}
There are two possible actions $y_i\in\{a,b\}$, three possible
rewards $r_i\in\{0,1,2\}$ and no observations.

Construct the environment $\nu_0$ as follows: if $y_i=a$ then
$r_i=1$ and if $y_i=b$ then $r_i=0$ for any $i\in\SetN$.

For each $i$ let $n_i$ denote the number of actions $a$ taken up to
step $i$: $n_i:=\#\{j\le i: y_j=a\}$. For each $s>0$ construct the
environment $\nu_s$ as follows: $r_i(a)=1$ for any $i$, $r_i(b)=2$
if the longest consecutive sequence of action $b$ taken has length
greater than $n_i$ and $n_i\ge s$; otherwise $r_i(b)=0$.

It is easy to see that each $\nu_i$, $i>0$ satisfies the value
stability conditions with $\phi(n,\eps)\equiv0$ except
$d(k,\eps)=k\ne o(k)$, and does not satisfy it with any
$d(k,\eps)=o(k)$. Next we show that there is no self-optimizing
policy for the class.

Suppose that there exists a policy $p$ such that $\low V(\nu_i,p) =
V^*_{\nu_i}$ for each $i>0$ and let the true environment be $\nu_0$.
By assumption, for each $s$ there exists such $n$ that
\beqn
  \#\{i\le n : y_i=b, r_i=0\}\ge s >\#\{i\le n: y_i=a, r_i=1\}
\eeqn
which implies $\low V(\nu_0,p)\le 1/2<1=V^*_{\nu_0}$.
\end{proof}

It is also easy to show that the {\em uniformity of convergence in
(\ref{eq:svs})} cannot be dropped. That is, if in the definition of
value-stability we allow the function $\phi(n,\eps)$ to depend
additionally on the past history $z_{<k}$ then
Theorem~\ref{th:main} does not hold. This can be shown with the same
example as constructed in the proof of Proposition~\ref{th:tight},
letting $d(k,\eps)\equiv0$ but instead allowing
$\phi(n,\eps,z_{<k})$ to take values 0 and 1 according to the number
of actions $a$ taken, achieving the same behaviour as in the
example provided in the last proof.

Moreover, we show that the requirement that the class $\mathcal C$
to be learnt is countable cannot be easily withdrawn. Indeed,
consider the class of all deterministic passive environments in the
sequence prediction setting. In this task an agent gets the reward
$1$ if $y_i=o_{i+1}$ and the reward $0$ otherwise, where the
sequence of observation $o_i$ is deterministic. Different sequences
correspond to different environments. As it was mentioned before,
any such environment $\nu$ is value-stable with
$d_\nu(k,\eps)\equiv1$, $\phi_\nu(n,\eps)\equiv0$ and
$r^\nu_i\equiv1$. Obviously, the class of all deterministic passive
environments is not countable. Since for every policy $p$ there is
an environment on which $p$ errs exactly on each step, the class of
all deterministic passive environments cannot be learned. Therefore,
the following statement is valid:
\begin{claim}
There exist (uncountable) classes of value-stable environments for
which there are no self-optimizing policies.
\end{claim}

However, strictly speaking, even for countable classes
value-stability is not necessary for self-optimizingness. This can
be demonstrated on the class $\nu_i: i>0$ from the proof of
Proposition~\ref{th:tight}. (Whereas if we add $\nu_0$ to the class
a self-optimizing policy no longer exists.) So we have the
following:
\begin{claim}
There are countable classes of not value-stable environments for which
self-optimizing policies exist.
\end{claim}

\section{Discussion}\label{secDisc}

\paradot{Summary}
We have proposed a set of conditions on environments, called
value-stability, such that any countable class of value-stable
environments admits a self-optimizing policy. It was also shown that
these conditions are in a certain sense tight. The class of all
value-stable environments includes ergodic MDPs, certain class of
finite POMDPs, passive environments, and (provably) more
environments.
So the concept of value-stability allows us to characterize
self-optimizing environment classes, and proving value-stability is
typically much easier than proving self-optimizingness directly.
Value stability means that from any (sup-optimal) sequence of
actions it is possibly to recover fast. If it is possible to
recover, but not necessarily fast, then we get a condition which we
called recoverability, which was shown to be sufficient to be able
to recover the upper limit of the optimal average asymptotic value.
We have also analyzed what can be achieved in environments which
possess (worst-case) value-stability but are not recoverable; it
turned out that a certain worst-case self-optimizingness can be
identified in this case too.

On the following picture we summarize the concepts introduced in
Sections~\ref{secSetup}, \ref{secMain} and~\ref{sec:nonrec}. The
arrows symbolize implications: some of them follow from theorems or
stated in definitions (marked accordingly), while others are
trivial.

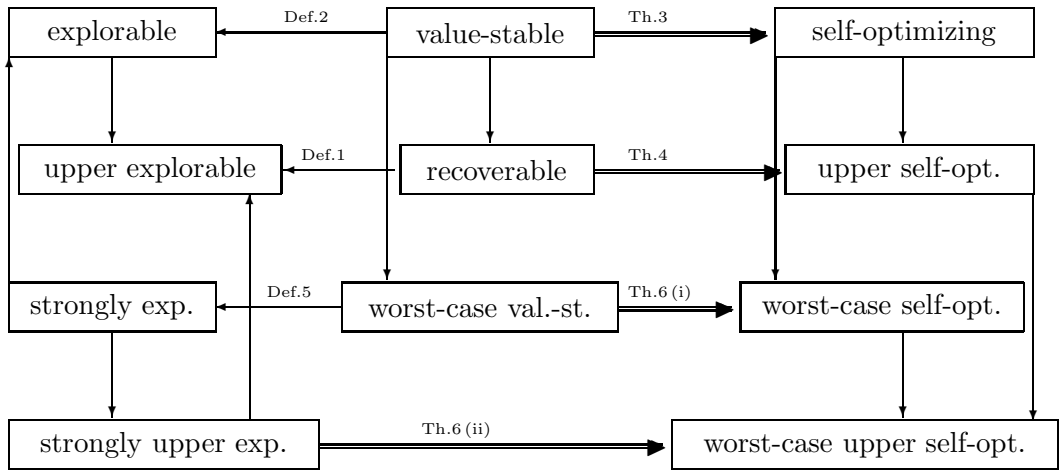
\begin{center}
\small
\unitlength=1.3pt
\begin{picture}(310,180)

\put(0,150){\framebox(60,14)[c]{explorable}}
\put(110,157){\vector(-1,0){50}}
\put(80,160){\text{{\tiny Def.\ref{def:vstable}}}}

\put(30,150){\vector(0,-1){25}}

\put(110,150){\framebox(60,14)[c]{value-stable}}
\put(170,157){\vector(1,0){50}}
\put(170,156){\vector(1,0){50}}
\put(215,154){\text{$\blacktriangleright$}}
\put(180,160){\text{{\tiny Th.\ref{th:main}}}}

\put(140,150){\vector(0,-1){25}}
\put(110,150){\vector(0,-1){65}}
\put(260,150){\vector(0,-1){25}}
\put(223,150){\vector(0,-1){65}}

\put(223,150){\framebox(75,14)[c]{self-optimizing}}

\put(3,110){\framebox(76,14)[c]{{\small upper explorable}}}
\put(112,117){\vector(-1,0){32}}
\put(85,120){\text{{\tiny Def.\ref{def:vstable2} }}}

\put(0,84){\vector(0,1){66}}
\put(298,110){\vector(0,-1){66}}

\put(114,110){\framebox(56,14)[c]{recoverable}}
\put(170,117){\vector(1,0){54}}
\put(170,116){\vector(1,0){54}}
\put(219,114){\text{$\blacktriangleright$}}

\put(180,120){\text{{\tiny Th.\ref{th:wr} }}}
\put(226,110){\framebox(72,14)[c]{upper self-opt.}}

\put(0,70){\framebox(60,14)[c]{{\small strongly exp.}}}
\put(97,77){\vector(-1,0){37}}
\put(75,80){\text{{\tiny Def.\ref{def:wstable} }}}

\put(30,70){\vector(0,-1){25}}
\put(260,70){\vector(0,-1){25}}

\put(97,70){\framebox(80,14)[c]{{\small worst-case val.-st.}}}
\put(177,77){\vector(1,0){33}}
\put(177,76){\vector(1,0){33}}
\put(205,74){\text{$\blacktriangleright$}}

\put(180,80){\text{{\tiny Th.\ref{th:worst}\,(i) }}}
\put(213,70){\framebox(82,14)[c]{{\small worst-case self-opt.}}}

\put(0,30){\framebox(90,14)[c]{{\small strongly upper exp.}}}

\put(70,44){\vector(0,1){66}}

\put(90,37){\vector(1,0){100}}
\put(90,36){\vector(1,0){100}}
\put(185,34){\text{$\blacktriangleright$}}

\put(120,40){\text{{\tiny Th.\ref{th:worst}\,(ii) }}}
\put(193,30){\framebox(112,14)[c]{{\small worst-case upper self-opt.}}}
\end{picture}
\end{center}

\paradot{Outlook}
We considered only countable environment classes $\C$. From a
computational perspective such classes are sufficiently large (e.g.\
the class of all computable probability measures is countable). On
the other hand, countability excludes continuously parameterized
families (like all ergodic MDPs), common in statistical practice.
So perhaps the main open problem is to find under which conditions
the requirement of countability of the class can be lifted. Another
important question is whether (meaningful) necessary and sufficient
conditions for self-optimizingness can be found. However,
identifying classes of environments for which self-optimizing
policies exist is a hard problem which has not been solved even for
passive environments \cite{RyabkoHutter:06pq}.

One more question concerns the uniformity of forgetfulness of the
environment. Currently in the definition of
value-stability~(\ref{eq:svs}) we have the function $\phi(n,\eps)$
which is the same for all histories $z_{<k}$, that is, both for all
actions histories $y_{<k}$ and observations-rewards histories
$x_{<k}$. Probably it is possible to differentiate between two types
of forgetfulness, one for actions and one for perceptions.

In this work we have chosen the asymptotic uniform average value
$\lim\odm\r1m^{p\nu}$ as our performance measure. Another popular
measure is the asymptotic discounted value $\gamma_1 r_1+\gamma_2
r_2+...$, where $\v\gamma$ is some (typically geometric
$\gamma_k\propto \gamma^k$) discount sequence. One can show
\cite{Hutter:06discount} under quite general conditions that the
limit of average and future discounted values coincide. Equivalence
holds for bounded rewards and monotone decreasing $\v\gamma$, in
deterministic environments and, in expectation over the history,
also for probabilistic environments. So, in these cases our results
also apply to discounted value.

Finally, it should be mentioned that we have concentrated on optimal
values which can be obtained with certainty (with probability one);
towards this aim we have defined (upper, strong) explorability and
only considered environments which possess one of these properties.
It would also be interesting to analyze what is achievable in
environments which are not (upper, strongly) explorable; for
example, one could consider optimal expected value, and may be some
other criteria.

\appendix
\section{Proofs of Theorems~\ref{th:main} and~\ref{th:wr}}\label{secPrThMain}
In each of the proofs, a self-optimizing (or upper self-optimizing)
policy $p$ will be constructed. When the policy $p$ has been defined
up to a step $k$, an environment $\mu$, endowed with this policy,
can be considered as a measure on $\Z^k$. We assume this meaning
when we use environments as measures on $\Z^k$ (e.g. $\mu(z_{<i})$).

\paradot{Proof of Theorem~\ref{th:main}}
A self-optimizing policy $p$ will be constructed as follows. On each
step we will have two polices: $p^t$ which exploits and $p^e$ which
explores; for each $i$ the policy $p$ either takes an action
according to $p^t$ ($p(z_{<i})=p^t(z_{<i})$) or according to $p^e$
($p(z_{<i})=p^e(z_{<i})$), as will be specified below.

In the algorithm below, $i$ denotes the number of the current step
in the sequence of actions-observations. Let $n=1$, $s=1$, and
$j^t=j^e=0$. Let also $\alpha_s=2^{-s}$ for $s\in\SetN$. For each
environment $\nu$, find such a sequence of real numbers $\eps^\nu_n$
that $\eps^\nu_n\rightarrow0$ and
$\sum_{n=1}^\infty\phi_\nu(n,\eps^\nu_n)\le\infty$.

Let $\i: \SetN\rightarrow\C$ be such a numbering that each
$\nu\in\C$ has infinitely many indices. For all $i>1$ define a
measure $\xi$ as follows
\beq\label{eq:xi}
  \xi(z_{<i})=\sum_{\nu\in\mathcal C}w_\nu\nu(z_{<i}),
\eeq
where $w_\nu\in\R$ are (any) such numbers that $\sum_{\nu}w_\nu=1$
and $w_\nu>0$ for all $\nu\in\mathcal C$.

\noindent{\bf\boldmath Define $T$.}
On each step $i$ let
\beqn
  T \;\equiv\; T_i \;:=\;
  \left\{\nu\in\C:\frac{\nu(z_{<i})}{\xi(z_{<i})}\ge\alpha_s\right\}
\eeqn

\noindent{\bf\boldmath Define $\nu^t$.} Set $\nu^t$ to be the first
environment in $T$ with index greater than $\i(j^t)$. In case this
is impossible (that is, if $T$ is empty), increment $s$, (re)define
$T$ and try again. Increment $j^t$.

\noindent{\bf\boldmath Define $\nu^e$.} Set $\nu^e$ to be the first
environment with index greater than $\i(j^e)$ such that
$V^*_{\nu^e}>V^*_{\nu^t}$ and $\nu^e(z_{<k})>0$, if such an
environment exists. Otherwise proceed one step (according to $p^t$)
and try again. Increment $j^e$.

\noindent{\bf Consistency.} On each step $i$ (re)define $T$. If
$\nu^t\notin T$, define $\nu^t$, increment $s$ and iterate the
infinite loop. (Thus $s$ is incremented only if $\nu^t$ is not in
$T$ or if $T$ is empty.)

Start the {\bf infinite loop}. Increment $n$.

Let $\delta:=(V^*_{\nu^e}-V^*_{\nu^t})/2$. Let
$\eps:=\eps^{\nu^t}_n$. If $\eps<\delta$ set $\delta=\eps$. Let
$h=j^e$.

\noindent{\bf Prepare for exploration.}

Increment $h$. The index $h$ is incremented with each next attempt
of exploring $\nu^e$. Each attempt will be at least $h$ steps in
length.

Let $p^t=p^{y_{<i}}_{\nu^t}$ and set $p=p^t$.

Let $i_h$ be the current step. Find $k_1$ such that
\beq\label{eq:k1}
  \frac{i_h}{k_1}V^*_{\nu^t} \le\eps/8
\eeq
Find $k_2>2i_h$ such that for all $m>k_2$
\beq\label{eq:k2}
  \left|\frac{1}{m-i_h} \r{i_h+1}{m}^{\nu^t}-V^*_{{\nu^t}}\right|\le \eps/8.
\eeq
Find $k_3$ such that
\beq\label{eq:k3}
  hr_{max}/k_3<\eps/8.
\eeq
Find $k_4$ such that for all $m>k_4$
\beq\label{eq:d}
  \frac{1}{m}d_{{\nu^e}}(m,\eps/4)\le
  \eps/8 \text{, \ \ }
  \frac{1}{m}d_{\nu^t}(m,\eps/8)\le \eps/8
 \text{\ \ and\ \ }
  \frac{1}{m}d_{\nu^t}(i_h,\eps/8)\le \eps/8.
\eeq
Moreover, it is always possible to find such
$k>\max\{k_1,k_2,k_3,k_4\}$ that
\beq\label{eq:k}
  \frac{1}{2k}\r{k}{3k}^{\nu^e} \ge \frac{1}{2k}\r{k}{3k}^{\nu^t} + \delta.
\eeq

Iterate up to the step $k$.

\noindent{\bf Exploration.}
Set $p^e=p_{{\nu^e}}^{y_{<n}}$. Iterate $h$ steps according to
$p=p^e$. Iterate further until either of the following conditions
breaks
\begin{itemize}
\item[$(i)$] $\left|\r{k}{i}^{\nu^e}-\r{k}{i}^{p\nu}\right|
              < (i-k)\eps/4+d_{\nu^e}(k,\eps/4)$,
\item[$(ii)$] $i<3k$.
\item[$(iii)$] $\nu^e\in T$.
\end{itemize}
Observe that either $(i)$ or $(ii)$ is necessarily broken.

If on some step $\nu^t$ is excluded from $T$ then the infinite loop
is iterated. If after exploration $\nu^e$ is not in $T$ then
redefine $\nu^e$ and {\bf iterate the infinite loop}. If both
$\nu^t$ and $\nu^e$ are still in $T$ then {\bf return} to ``Prepare
for exploration'' (otherwise the loop is iterated with either
$\nu^t$ or $\nu^e$ changed).

\noindent{\bf\boldmath End} of the infinite loop and the
algorithm.

Let us show that with probability $1$ the ``Exploration'' part is
iterated only a finite number of times in a row with the same
$\nu^t$ and $\nu^e$.

Suppose the contrary, that is, suppose that (with some non-zero
probability) the ``Exploration'' part is iterated infinitely often
while $\nu^t,\nu^e\in T$. Observe that (\ref{eq:svs}) implies that
the $\nu^e$-probability that $(i)$ breaks is not greater than
$\phi_{\nu_e}(i-k,\eps/4)$; hence by Borel-Cantelli lemma the event
that $(i)$ breaks infinitely often has probability 0 under $\nu^e$.

Suppose that $(i)$ holds almost every time. Then $(ii)$ should be
broken except for a finite number of times. We can use
(\ref{eq:k1}), (\ref{eq:k2}), (\ref{eq:d}) and (\ref{eq:k}) to show
that with probability at least $1-\phi_{\nu^t}(k-i_h,\eps/4)$ under
$\nu^t$ we have $\frac{1}{3k}\r{1}{3k}^{p\nu^t}\ge
V^*_{\nu^t}+\eps/2$. Again using Borel-Cantelli lemma and $k>2i_h$
we obtain that the event that $(ii)$ breaks infinitely often has
probability $0$ under $\nu^t$.

Thus (at least) one of the environments $\nu^t$ and $\nu^e$ is
singular with respect to the true environment $\nu$ given the
described policy and current history. Denote this environment by
$\nu'$. It is known (see e.g. \cite[Thm.26]{CsiszarShields:04}) that
if measures $\mu$ and $\nu$ are mutually singular then
$\frac{\mu(x_1,\dots,x_n)}{\nu(x_1,\dots,x_n)}\rightarrow\infty$
$\mu$-a.s. Thus
\beq\label{eq:sing}
  \frac{\nu'(z_{<i})}{\nu(z_{<i})}\rightarrow0\text{ $\nu$-a.s.}
\eeq
Observe that (by definition of $\xi$) $\frac{\nu(z_{<i})}{\xi(z_{<i})}$ is
bounded. Hence using (\ref{eq:sing}) we can see that
\beqn
  \frac{\nu'(z_{<i})}{\xi(z_{<i})}\rightarrow0\text{ $\nu$-a.s.}
\eeqn
Since $s$ and $\alpha_s$ are not changed during the exploration
phase this implies that on some step $\nu'$ will be excluded from
$T$ according to the ``consistency'' condition, which contradicts
the assumption. Thus the ``Exploration'' part is iterated only a
finite number of times in a row with the same $\nu^t$ and $\nu^e$.

Observe that $s$ is incremented only a finite number of times since
$\frac{\nu'(z_{<i})}{\xi(z_{<i})}$ is bounded away from $0$ where
$\nu'$ is either the true environment $\nu$ or any environment from
$\C$ which is equivalent to $\nu$ on the current history. The latter
follows from the fact that $\frac{\xi(z_{<i})}{\nu(z_{<i})}$ is a
submartingale with bounded expectation, and hence, by the
submartingale convergence theorem (see e.g. \cite{Doob:53})
converges with $\nu$-probability 1.

Let us show that from some step on $\nu$ (or an environment
equivalent to it) is always in $T$ and selected as $\nu^t$. Consider
the environment $\nu^t$ on some step $i$. If $V^*_{\nu^t}>V^*_\nu$
then $\nu^t$ will be excluded from $T$ since on any optimal for
$\nu^t$ sequence of actions (policy) measures $\nu$ and $\nu^t$ are
singular. If $V^*_{\nu^t}<V^*_\nu$ than $\nu^e$ will be equal to
$\nu$ at some point, and, after this happens sufficient number of
times, $\nu^t$ will be excluded from $T$ by the ``exploration'' part
of the algorithm, $s$ will be decremented and $\nu$ will be
included into $T$. Finally, if $V^*_{\nu^t}=V^*_\nu$ then either the
optimal value $V^*_\nu$ is (asymptotically) attained by the policy
$p_t$ of the algorithm, or (if $p_{\nu^t}$ is suboptimal for $\nu$)
$\frac{1}{i} \r1i^{p{\nu^t}}< V^*_{\nu^t}-\eps$ infinitely often for
some $\eps$, which has probability $0$ under $\nu^t$ and
consequently $\nu^t$ is excluded from $T$.

Thus, the exploration part ensures that all environments not
equivalent to $\nu$ with indices smaller than $\i(\nu)$ are removed
from $T$ and so from some step on $\nu^t$ is equal to (an
environment equivalent to) the true environment $\nu$.

We have shown in the ``Exploration'' part that $n\rightarrow\infty$,
and so $\eps^{\nu^t}_n\rightarrow0$. Finally, using the same
argument as before (Borel-Cantelli lemma, $(i)$ and the definition
of $k$) we can show that in the ``exploration'' and ``prepare for
exploration'' parts of the algorithm the average value is within
$\eps^{\nu^t}_n$ of $V^*_{\nu^t}$ provided the true environment is
(equivalent to) $\nu^t$. \hspace*{\fill}$\Box\quad$

\paradot{Proof of Theorem~\ref{th:wr}}
Let $\i: \SetN\rightarrow\C$ be such a numbering that each
$\nu\in\C$ has infinitely many indices. Define the measure $\xi$ as
in~(\ref{eq:xi}). The policy $p$ acts according to the following
algorithm.

Set $\eps_s =\alpha_s=2^{-s}$ for $s\in\SetN$, set $j=1$, $s=n=1$.
The integer $i$ will denote the current step in time.

Do the following \emph{ad infinitum}. Set $\nu$ to be the first
environment in $\C$ with index greater than $\i(j)$. Find the policy
$p_\nu$ which achieves the upper limiting average value with
probability one (such policy exists by definition of
recoverability). Act according to $p_\nu$ until either
\beq \label{eq:cond1}
  \left|{1\over i} r_{1..i}^{p\nu} - \up V^*(p,p_\nu)\right| <\eps_n
\eeq
or
\beq \label{eq:cond2}
  \frac{\nu(z_{<i})}{\xi(z_{<i})} < \alpha_s.
\eeq
Increment $n$, $s$, $i$.

It can be easily seen that one of the conditions necessarily breaks.
Indeed, either in the true environment the optimal upper limiting
average value for the current environment $\nu$ can be achieved by
the optimal policy $p_\nu$, in which case (\ref{eq:cond1}) breaks;
or it cannot be achieved, which means that $\nu$ and $\xi$ are
singular, which implies that (\ref{eq:cond2}) will be broken (see
e.g. \cite[Thm.26]{CsiszarShields:04}; cf. the same argument in the
proof of Theorem~\ref{th:main}). Since $\nu$ equals the true
environment infinitely often and $\eps_n\rightarrow 0$ we get the
statement of the theorem. \qed

\paranodot{Proof of Theorem~\ref{th:worst}} is analogous to the
proofs of Theorems~\ref{th:main} and~\ref{th:wr}, except for the
following. Instead of the optimal average value $V^*_\nu$ and upper
optimal average value $\up V^*_\nu$ the values $V^*_\nu(z_{<k})$
and $\up V^*_\nu(z_{<k})$ should be used, and they should be updated
after each step $k$. \qed


\begin{small}
\bibliographystyle{alpha}

\end{small}

\end{document}